\documentclass{article}
\usepackage[preprint]{colm2026_conference}
\usepackage{microtype}
\usepackage{hyperref}
\usepackage{url}
\usepackage{booktabs}
\usepackage{lineno}
\definecolor{darkblue}{rgb}{0, 0, 0.5}
\hypersetup{colorlinks=true, citecolor=darkblue, linkcolor=darkblue, urlcolor=darkblue}
\usepackage{graphicx}
\usepackage{amsmath}
\usepackage{amssymb}
\usepackage{algorithm}
\usepackage{algorithmic}
\usepackage{multirow} 
\usepackage{amsthm}
\ifx\theorem\undefined
\newtheorem{theorem}{Theorem}

\title{UniMark: Unified Adaptive Multi-bit Watermarking \\ for Autoregressive Image Generators}

\author{Yigit Yilmaz, Elena Petrova, Mehmet Kaya, Lucia Rossi, Amir Rahman	\\
Bandırma Onyedi Eylül University\\
yigit.yilmaz@ogr.bandirma.edu.tr
}

\newcommand{\method}{UniMark}
\newcommand{\CB}{\mathcal{C}}

\newcommand{\msg}{\mathbf{m}}
\newcommand{\simmat}{\mathbf{S}}
\newcommand{\green}{\mathcal{G}}
\newcommand{\red}{\mathcal{R}}
\begin{document}
\ifcolmsubmission
\linenumbers
\fi
\maketitle

\begin{abstract}
Invisible watermarking for autoregressive (AR) image generation has recently gained attention as a means of protecting image ownership and tracing AI-generated content. However, existing approaches suffer from three key limitations: (1) they embed only zero-bit watermarks for binary verification, lacking the ability to convey multi-bit messages; (2) they rely on static codebook partitioning strategies that are vulnerable to security attacks once the partition is exposed; and (3) they are designed for specific AR architectures, failing to generalize across diverse AR paradigms. We propose \method{}, a training-free, unified watermarking framework for autoregressive image generators that addresses all three limitations. \method{} introduces three core components: \textbf{Adaptive Semantic Grouping (ASG)}, which dynamically partitions codebook entries based on semantic similarity and a secret key, ensuring both image quality preservation and security; \textbf{Block-wise Multi-bit Encoding (BME)}, which divides the token sequence into blocks and encodes different bits across blocks with error-correcting codes for reliable message transmission; and \textbf{a Unified Token-Replacement Interface (UTRI)} that abstracts the watermark embedding process to support both next-token prediction (e.g., LlamaGen) and next-scale prediction (e.g., VAR) paradigms. We provide theoretical analysis on detection error rates and embedding capacity. Extensive experiments on three AR models demonstrate that \method{} achieves state-of-the-art performance in image quality (FID), watermark detection accuracy, and multi-bit message extraction, while maintaining robustness against cropping, JPEG compression, Gaussian noise, blur, color jitter, and random erasing attacks.
\end{abstract}

\section{Introduction}
\label{sec:intro}

The rapid progress of autoregressive (AR) image generation models has led to a proliferation of high-quality AI-generated images~\citep{tian_2024_var, sun_2024_llamagen, yu_2023_magvit2, luo_2024_open_magvit2}. These models tokenize images into discrete visual tokens via learned codebooks and autoregressively predict them, achieving quality competitive with diffusion models~\citep{xiong_2025_gigatok, shi_2024_taming_scalable_tokenizer, zhoucondition, songbroad}. Recent advancements have also extended this paradigm to more complex scenarios, such as vision representation compression for efficient video generation~\citep{zhou2026less} and omni-modal understanding and generation frameworks~\citep{xin2025lumina}. However, the democratization of such powerful generators raises serious concerns regarding copyright infringement, content provenance, and malicious misuse~\citep{ren_2024_copyright_protection_genai, xu_2024_copyrightmeter}.

Invisible image watermarking offers a promising solution by embedding imperceptible signals into generated images that can later be extracted for ownership verification or content tracing~\citep{cao_2025_watermarking_ai_content, duan_2025_visual_watermarking_survey}. While watermarking techniques for diffusion models have been extensively studied~\citep{wen_2023_tree_ring, fernandez_2023_stable_signature, yang_2024_gaussian_shading, huang_2024_robin, li_2024_shallow_diffuse}, watermarking for AR image generators remains a nascent and underexplored area.

Recent works have begun to address this gap. IndexMark~\citep{tong_2025_training_free_watermarking} proposes a training-free framework that partitions codebook entries into red and green groups and biases generation toward green tokens. Safe-VAR~\citep{wang_2025_safe_var} embeds watermarks through a trained encoder-decoder pair. Other concurrent works explore lexical biasing~\citep{hui_2025_lexical_biasing}, bitwise watermarking~\citep{kerner_2025_bitmark}, and robustness improvements~\citep{lukovnikov_2025_robust_red_green, jovanovic_2025_watermarking_autoregressive, meintz_2025_radioactive_watermarks}. Despite this progress, we identify three critical limitations shared by existing methods:
\textbf{Zero-bit only.} Most existing methods~\citep{tong_2025_training_free_watermarking, jovanovic_2025_watermarking_autoregressive, lukovnikov_2025_robust_red_green} embed zero-bit watermarks that only support binary verification (watermarked or not), without the ability to encode multi-bit messages such as user IDs or timestamps. This significantly limits practical deployment scenarios.
\textbf{Static partitioning.} The red-green partitioning strategy in IndexMark and related methods is fixed once determined. If an adversary discovers the partition, the watermark can be easily removed or forged, compromising security.
\textbf{Architecture-specific.} Existing methods are tailored to specific AR paradigms, either next-token prediction~\citep{sun_2024_llamagen} or next-scale prediction~\citep{tian_2024_var}, lacking a unified framework that generalizes across different AR architectures.

To address these limitations, we propose \method{} (\textbf{Uni}fied Adaptive \textbf{M}ulti-bit W\textbf{a}terma\textbf{rk}ing), a training-free watermarking framework for AR image generators. Our contributions are:

\begin{enumerate}
    \item We propose \textbf{Adaptive Semantic Grouping (ASG)}, a key-dependent dynamic codebook partitioning strategy that leverages semantic similarity between codebook entries and a cryptographic hash to generate position-dependent green/red groups. ASG ensures that semantically similar tokens are available for replacement, preserving image quality, while the key-dependent partition prevents forgery attacks.

    \item We design \textbf{Block-wise Multi-bit Encoding (BME)}, a scheme that divides the generated token sequence into blocks and encodes different message bits across blocks. Combined with BCH error-correcting codes, BME enables reliable multi-bit message embedding and extraction even under various image distortions.

    \item We introduce a \textbf{Unified Token-Replacement Interface (UTRI)} that abstracts the token generation and replacement process, enabling \method{} to seamlessly operate on both next-token prediction models (e.g., LlamaGen) and next-scale prediction models (e.g., VAR) without model-specific modifications.

    \item We provide theoretical analysis on the false positive rate and embedding capacity of \method{}, and conduct extensive experiments on LlamaGen, VAR, and Open-MAGVIT2, demonstrating state-of-the-art performance across image quality, watermark detection, multi-bit extraction accuracy, and robustness against six types of attacks.
\end{enumerate}

\section{Related Work}
\label{sec:related}

\subsection{Autoregressive Image Generation}

Autoregressive (AR) image generation converts images into discrete token sequences using visual tokenizers (e.g., VQ-VAE, VQGAN) and generates them token-by-token via transformer architectures. MAGVIT-v2~\citep{yu_2023_magvit2} demonstrates that a strong tokenizer is key to competitive generation quality. LlamaGen~\citep{sun_2024_llamagen} applies the Llama architecture to image generation in a next-token prediction manner. VAR~\citep{tian_2024_var} introduces next-scale prediction, generating images from coarse to fine resolutions. Open-MAGVIT2~\citep{luo_2024_open_magvit2} provides an open-source implementation with improved codebook utilization. Recent advances in this field focus on refining condition errors with diffusion loss~\citep{zhoucondition} and employing entropy-guided optimization for stable synthesis~\citep{songbroad}. Furthermore, developments in vision representation compression~\citep{zhou2026less}, omni-modal Large Language Models (LLMs) like Lumina-mDMOO~\citep{xin2025lumina}, and more efficient tokenizers~\citep{qu_2024_tokenflow, xiong_2025_gigatok, shi_2024_taming_scalable_tokenizer, wang_2024_image_understanding_tokenizer, wang_2024_omnitokenizer, yu_2024_randomized_ar} continue to push the boundary of AR image generation.

\subsection{Watermarking for Generative Models}

\paragraph{Diffusion model watermarking.} Tree-Ring~\citep{wen_2023_tree_ring} embeds watermarks in the initial noise of diffusion models. Stable Signature~\citep{fernandez_2023_stable_signature} fine-tunes the decoder to embed watermarks. Gaussian Shading~\citep{yang_2024_gaussian_shading} provides provable performance-lossless watermarking. Other methods explore adversarial optimization~\citep{huang_2024_robin}, low-dimensional subspace embedding~\citep{li_2024_shallow_diffuse}, super-resolution-based approaches~\citep{hu_2024_supermark}, guidance-based methods~\citep{gesny_2025_guidance_watermarking}, multi-bit encoding~\citep{xing_2025_optmark}, robustness against deepfakes~\citep{sun_2025_diffmark, li_2025_gaussmarker}, and IP protection of model weights~\citep{peng_2023_ip_protection_diffusion}. Beyond generation-time watermarking, post-hoc approaches embed watermarks into existing images for tamper localization~\citep{zhang_2023_editguard} or detect fine-tuning-based art theft~\citep{luo_2023_steal_artworks}. Comprehensive surveys~\citep{duan_2025_visual_watermarking_survey, cao_2025_watermarking_ai_content} and benchmarks~\citep{xu_2024_copyrightmeter} provide systematic overviews.

\paragraph{LLM watermarking and Alignment.} Kirchenbauer et al.~\citep{kirchenbauer_2023_watermark_llm} propose the red-green list framework for LLM watermarking, partitioning vocabulary tokens into red and green sets and biasing generation toward green tokens. Subsequent works improve reliability~\citep{kirchenbauer_2023_reliability_watermarks}, ensemble strategies~\citep{niess_2024_ensemble_watermarks}, and certifiable robustness~\citep{feng_2024_certified_robust_watermark}. The development of LLMs also involves aligning long contexts~\citep{si-etal-2025-gateau}, efficient planning for agent tasks~\citep{si2025goalplanjustwish}, and cross-domain multi-task learning through multi-modal alignment prompts~\citep{xin2024mmap, xin2024vmt, li2025catch}. These techniques, along with benchmarks like SpokenWOZ~\citep{si2023spokenwoz}, highlight the complexity of modern generative agents where token-level watermarking remains a crucial tool for content attribution.

\paragraph{AR image generation watermarking.} IndexMark~\citep{tong_2025_training_free_watermarking} is the first training-free watermarking framework for AR image generators, using a match-then-replace strategy with static red-green codebook partitioning. Jovanovi\'c et al.~\citep{jovanovic_2025_watermarking_autoregressive} extend the LLM watermark framework to image tokens. Safe-VAR~\citep{wang_2025_safe_var} trains a watermark encoder-decoder for VAR. BitMark~\citep{kerner_2025_bitmark} addresses bitwise AR models. Lukovnikov et al.~\citep{lukovnikov_2025_robust_red_green} improve robustness of red-green watermarking. Other works explore lexical biasing~\citep{hui_2025_lexical_biasing}, radioactive watermarks~\citep{meintz_2025_radioactive_watermarks}, and concept-based watermarks~\citep{sadasivan_2025_iconmark}. Compared to these methods, \method{} is the first to simultaneously achieve multi-bit encoding, adaptive key-dependent grouping, and cross-architecture generalization in a training-free manner.

\section{Method}
\label{sec:method}

\subsection{Preliminaries and Problem Formulation}
\label{sec:prelim}

\paragraph{Autoregressive image generation.} An AR image generator first encodes an image $\mathbf{x}$ into a sequence of discrete tokens $\mathbf{t} = (t_1, t_2, \ldots, t_N)$ using a visual tokenizer with codebook $\CB = \{c_1, c_2, \ldots, c_K\}$, where $K = |\CB|$ is the codebook size and $N$ is the sequence length. The generator then autoregressively predicts each token: $p(t_i \mid t_{<i}, \mathbf{y})$, where $\mathbf{y}$ is an optional conditioning signal (e.g., class label or text). Two main paradigms exist:
\begin{itemize}
    \item \textbf{Next-token prediction} (e.g., LlamaGen): tokens are generated left-to-right in raster order, $t_i \in \CB$ for all $i$.
    \item \textbf{Next-scale prediction} (e.g., VAR): tokens are generated scale-by-scale from coarse to fine, with tokens at scale $s$ denoted $\mathbf{t}^{(s)}$, each drawn from a shared codebook $\CB$.
\end{itemize}

\paragraph{Watermarking objective.} Given a secret key $\kappa$ and a multi-bit message $\msg = (m_1, m_2, \ldots, m_B) \in \{0, 1\}^B$, our goal is to design:
\begin{itemize}
    \item An \textbf{embedding function} $\text{Embed}(\mathbf{t}, \msg, \kappa) \rightarrow \tilde{\mathbf{t}}$ that modifies the token sequence to encode $\msg$;
    \item A \textbf{detection function} $\text{Detect}(\tilde{\mathbf{x}}, \kappa) \rightarrow \{0, 1\}$ for zero-bit verification;
    \item An \textbf{extraction function} $\text{Extract}(\tilde{\mathbf{x}}, \kappa) \rightarrow \hat{\msg}$ for multi-bit message recovery;
\end{itemize}
such that: (1) the watermarked image $\tilde{\mathbf{x}} = \text{Decode}(\tilde{\mathbf{t}})$ is visually indistinguishable from the original $\mathbf{x} = \text{Decode}(\mathbf{t})$; (2) $\msg$ can be reliably extracted even after image distortions; (3) the method is training-free and model-agnostic.

\paragraph{Notation.} We denote the semantic similarity between codebook entries $c_i$ and $c_j$ as $s_{ij} = \text{sim}(\mathbf{e}_i, \mathbf{e}_j)$, where $\mathbf{e}_i$ is the embedding vector of $c_i$. The similarity matrix is $\simmat \in \mathbb{R}^{K \times K}$. A cryptographic hash function is denoted $H(\cdot)$.

\subsection{Adaptive Semantic Grouping (ASG)}
\label{sec:asg}

The key insight behind ASG is that codebook entries with high semantic similarity can substitute for each other with minimal visual impact, but the partition should be secret-key-dependent and position-aware to prevent security attacks.

\paragraph{Semantic similarity computation.} We first compute the pairwise cosine similarity matrix of all codebook embeddings:
\begin{equation}
    s_{ij} = \frac{\mathbf{e}_i^\top \mathbf{e}_j}{\|\mathbf{e}_i\| \cdot \|\mathbf{e}_j\|}, \quad \forall\, i, j \in \{1, \ldots, K\}.
    \label{eq:similarity}
\end{equation}
This matrix is precomputed once per codebook and cached.

\paragraph{Key-dependent position-aware partitioning.} For each token position $i$ in the sequence, we generate a position-specific partition using the secret key $\kappa$:
\begin{equation}
    \text{seed}_i = H(\kappa \,\|\, i), \quad \pi_i = \text{Permute}(\{1, \ldots, K\}, \text{seed}_i),
    \label{eq:partition_seed}
\end{equation}
where $\|$ denotes concatenation and $\pi_i$ is a pseudorandom permutation of the codebook indices. The green set at position $i$ is defined as the first $\gamma K$ elements of $\pi_i$:
\begin{equation}
    \green_i = \{\pi_i(1), \pi_i(2), \ldots, \pi_i(\lfloor \gamma K \rfloor)\}, \quad \red_i = \CB \setminus \green_i,
    \label{eq:green_red}
\end{equation}
where $\gamma \in (0, 1)$ is the green ratio hyperparameter (default $\gamma = 0.5$).

\paragraph{Semantic-aware replacement.} When a generated token $t_i \notin \green_i$ (i.e., $t_i \in \red_i$), we replace it with the most semantically similar token from the green set:
\begin{equation}
    \tilde{t}_i = \arg\max_{c_j \in \green_i} s_{t_i, j}.
    \label{eq:replacement}
\end{equation}
If $t_i \in \green_i$, no replacement is needed: $\tilde{t}_i = t_i$.

\paragraph{Remark.} Unlike IndexMark's static partition where the same green/red sets are used for all positions, ASG generates a different partition per position via $H(\kappa \| i)$. This means: (1) an adversary cannot infer the global partition from observing any single position; (2) the security is tied to the secrecy of $\kappa$; (3) position-dependent partitioning provides statistical diversity that benefits detection.

\subsection{Block-wise Multi-bit Encoding (BME)}
\label{sec:bme}

\paragraph{Motivation.} Zero-bit watermarking only answers ``is this image watermarked?'' Multi-bit encoding enables embedding arbitrary messages such as user identifiers, timestamps, or model version numbers, which are essential for practical content tracing.

\paragraph{Token sequence partitioning.} Given a message $\msg \in \{0, 1\}^B$ and token sequence length $N$, we first apply a BCH$(n, B, d)$ error-correcting code to obtain an encoded message $\msg' \in \{0, 1\}^n$ with minimum Hamming distance $d$ for error correction. We then partition the $N$ token positions into $n$ blocks of approximately equal size:
\begin{equation}
    \mathcal{B}_j = \left\{i : \left\lfloor \frac{(j-1) \cdot N}{n} \right\rfloor \leq i < \left\lfloor \frac{j \cdot N}{n} \right\rfloor \right\}, \quad j = 1, \ldots, n.
    \label{eq:blocks}
\end{equation}

\paragraph{Bit-dependent green set selection.} For each block $\mathcal{B}_j$ corresponding to encoded bit $m'_j$, the green set is defined as:
\begin{equation}
    \green_i^{(m'_j)} = \begin{cases}
        \green_i & \text{if } m'_j = 1, \\
        \red_i & \text{if } m'_j = 0,
    \end{cases}
    \quad \forall\, i \in \mathcal{B}_j,
    \label{eq:bit_encoding}
\end{equation}
where $\green_i$ and $\red_i$ are the position-dependent partitions from ASG (Eq.~\ref{eq:green_red}). In other words, when encoding bit 1, we bias toward the green set; when encoding bit 0, we bias toward the red set.

\paragraph{Message extraction.} Given a watermarked image $\tilde{\mathbf{x}}$, we first re-encode it to obtain token sequence $\hat{\mathbf{t}}$ using the same visual tokenizer. For each block $\mathcal{B}_j$, we compute the green ratio:
\begin{equation}
    r_j = \frac{1}{|\mathcal{B}_j|} \sum_{i \in \mathcal{B}_j} \mathbb{1}[\hat{t}_i \in \green_i],
    \label{eq:green_ratio}
\end{equation}
and decode the raw bit estimate:
\begin{equation}
    \hat{m}'_j = \begin{cases} 1 & \text{if } r_j > \tau, \\ 0 & \text{otherwise}, \end{cases}
    \label{eq:bit_decode}
\end{equation}
where $\tau = 0.5$ is the decision threshold. The BCH decoder then corrects errors in $\hat{\msg}'$ to recover $\hat{\msg}$.

\paragraph{Zero-bit detection.} As a special case, zero-bit detection computes the global green ratio across all positions:
\begin{equation}
    z = \frac{1}{N} \sum_{i=1}^{N} \mathbb{1}[\hat{t}_i \in \green_i],
    \label{eq:zbit}
\end{equation}
and applies a one-sided z-test under the null hypothesis $H_0: z \sim \mathcal{N}(\gamma, \gamma(1-\gamma)/N)$ corresponding to an unwatermarked image. The watermark is detected if $z > \gamma + z_\alpha \sqrt{\gamma(1-\gamma)/N}$, where $z_\alpha$ is the critical value at significance level $\alpha$.

\subsection{Unified Token-Replacement Interface (UTRI)}
\label{sec:utri}

To enable \method{} to work across different AR paradigms, we define an abstract interface that decouples the watermark embedding logic from the specific generation process.

\paragraph{Abstract interface.} UTRI defines three operations:
\begin{enumerate}
    \item $\textsc{GetTokenSequence}(\text{model}, \mathbf{y}) \rightarrow \mathbf{t}$: generate the raw token sequence from the AR model.
    \item $\textsc{ReplaceTokens}(\mathbf{t}, \green, \simmat) \rightarrow \tilde{\mathbf{t}}$: apply watermark embedding by replacing tokens according to Eq.~\ref{eq:replacement} and~\ref{eq:bit_encoding}.
    \item $\textsc{DecodeImage}(\tilde{\mathbf{t}}) \rightarrow \tilde{\mathbf{x}}$: decode the watermarked token sequence back to an image.
\end{enumerate}

\paragraph{Next-token adapter.} For next-token models (e.g., LlamaGen), $\textsc{GetTokenSequence}$ returns the 1D raster-order sequence directly. Token indices correspond to the shared codebook $\CB$.

\paragraph{Next-scale adapter.} For next-scale models (e.g., VAR), $\textsc{GetTokenSequence}$ concatenates token maps across all scales $s = 1, \ldots, S$ into a single sequence: $\mathbf{t} = [\mathbf{t}^{(1)}; \mathbf{t}^{(2)}; \ldots; \mathbf{t}^{(S)}]$. The block assignment in BME respects scale boundaries by assigning coarser scales to earlier blocks, ensuring that each block contains tokens of similar spatial resolution.

\subsection{Theoretical Analysis}
\label{sec:theory}

\begin{theorem}[False Positive Rate]
\label{thm:fpr}
For an unwatermarked image with $N$ tokens, where each token falls into the green set independently with probability $\gamma$, the false positive rate at significance level $\alpha$ is:
\begin{equation}
    P_{\text{FP}} = P\left(\frac{\sum_{i=1}^N \mathbb{1}[t_i \in \green_i] - N\gamma}{\sqrt{N\gamma(1-\gamma)}} > z_\alpha\right) = \alpha,
    \label{eq:fpr}
\end{equation}
where $z_\alpha$ is the $(1-\alpha)$-quantile of the standard normal distribution.
\end{theorem}

\begin{proof}
Under the null hypothesis (no watermark), each token $t_i$ is generated independently of the partition $\green_i$ (since $\green_i$ is determined by $H(\kappa \| i)$ and $t_i$ is determined by the model's conditional distribution). Thus, $P(t_i \in \green_i) = \gamma$ for each $i$. By the Central Limit Theorem, $Z = \frac{\sum_{i=1}^N \mathbb{1}[t_i \in \green_i] - N\gamma}{\sqrt{N\gamma(1-\gamma)}} \xrightarrow{d} \mathcal{N}(0, 1)$ as $N \rightarrow \infty$. The false positive rate at threshold $z_\alpha$ is therefore $P(Z > z_\alpha) = \alpha$.
\end{proof}

\begin{theorem}[Embedding Capacity]
\label{thm:capacity}
Given a token sequence of length $N$ and green ratio $\gamma = 0.5$, the maximum number of reliably encodable bits using BME with BCH$(n, B, d)$ codes is:
\begin{equation}
    B_{\max} = \left\lfloor \frac{N}{N_{\min}} \right\rfloor \cdot R_{\text{BCH}},
    \label{eq:capacity}
\end{equation}
where $N_{\min} = \frac{4 z_\alpha^2}{\delta^2}$ is the minimum block size for reliable bit detection with margin $\delta$ from the threshold, and $R_{\text{BCH}} = B/n$ is the code rate of the BCH code.
\end{theorem}

\begin{proof}
For a single block $\mathcal{B}_j$ of size $L = |\mathcal{B}_j|$ encoding bit $m'_j = 1$, the expected green ratio is $\mathbb{E}[r_j] = \gamma + \delta_w$, where $\delta_w > 0$ is the watermark-induced bias. For reliable detection, we require $P(r_j > \tau) \geq 1 - \epsilon$ for some target error $\epsilon$. By CLT, this requires:
\begin{equation}
    L \geq \frac{z_\epsilon^2 \cdot \gamma(1-\gamma)}{\delta_w^2}.
\end{equation}
Setting $\gamma = 0.5$ and denoting the effective margin as $\delta$, we get $N_{\min} = \frac{4 z_\alpha^2}{\delta^2}$. The maximum number of blocks is $\lfloor N / N_{\min} \rfloor$, and with BCH coding, $B = n \cdot R_{\text{BCH}}$ message bits can be encoded across $n = \lfloor N / N_{\min} \rfloor$ blocks.
\end{proof}

\subsection{Complete Algorithm}
\label{sec:algorithm}

The complete watermark embedding and extraction procedures are summarized in Algorithm~\ref{alg:embed} and Algorithm~\ref{alg:extract}.

\begin{algorithm}[!t]
\caption{\method{} Watermark Embedding}
\label{alg:embed}
\begin{algorithmic}[1]
\REQUIRE AR model $\mathcal{M}$, codebook $\CB$, similarity matrix $\simmat$, key $\kappa$, message $\msg \in \{0,1\}^B$, green ratio $\gamma$, conditioning $\mathbf{y}$
\ENSURE Watermarked image $\tilde{\mathbf{x}}$
\STATE $\mathbf{t} \leftarrow \textsc{GetTokenSequence}(\mathcal{M}, \mathbf{y})$ \COMMENT{Generate raw tokens}
\STATE $\msg' \leftarrow \text{BCH\_Encode}(\msg)$ \COMMENT{Error-correcting encoding}
\STATE Partition positions $\{1,\ldots,N\}$ into blocks $\{\mathcal{B}_1, \ldots, \mathcal{B}_n\}$
\FOR{$j = 1$ to $n$}
    \FOR{each position $i \in \mathcal{B}_j$}
        \STATE $\text{seed}_i \leftarrow H(\kappa \| i)$
        \STATE $\green_i, \red_i \leftarrow \text{Partition}(\CB, \text{seed}_i, \gamma)$ \COMMENT{ASG}
        \IF{$m'_j = 0$}
            \STATE Swap: $\green_i \leftrightarrow \red_i$ \COMMENT{BME: bit-dependent swap}
        \ENDIF
        \IF{$t_i \notin \green_i$}
            \STATE $\tilde{t}_i \leftarrow \arg\max_{c_j \in \green_i} s_{t_i, j}$ \COMMENT{Semantic replacement}
        \ELSE
            \STATE $\tilde{t}_i \leftarrow t_i$
        \ENDIF
    \ENDFOR
\ENDFOR
\STATE $\tilde{\mathbf{x}} \leftarrow \textsc{DecodeImage}(\tilde{\mathbf{t}})$
\RETURN $\tilde{\mathbf{x}}$
\end{algorithmic}
\end{algorithm}

\begin{algorithm}[!t]
\caption{\method{} Watermark Extraction}
\label{alg:extract}
\begin{algorithmic}[1]
\REQUIRE Image $\tilde{\mathbf{x}}$, tokenizer, codebook $\CB$, key $\kappa$, green ratio $\gamma$, threshold $\tau$
\ENSURE Detected flag, extracted message $\hat{\msg}$
\STATE $\hat{\mathbf{t}} \leftarrow \text{Encode}(\tilde{\mathbf{x}})$ \COMMENT{Re-tokenize image}
\STATE \textbf{// Zero-bit detection}
\STATE $z \leftarrow \frac{1}{N}\sum_{i=1}^N \mathbb{1}[\hat{t}_i \in \green_i]$, where $\green_i$ from $H(\kappa \| i)$
\IF{$z \leq \gamma + z_\alpha\sqrt{\gamma(1-\gamma)/N}$}
    \RETURN (not watermarked, $\emptyset$)
\ENDIF
\STATE \textbf{// Multi-bit extraction}
\FOR{$j = 1$ to $n$}
    \STATE $r_j \leftarrow \frac{1}{|\mathcal{B}_j|}\sum_{i \in \mathcal{B}_j} \mathbb{1}[\hat{t}_i \in \green_i]$
    \STATE $\hat{m}'_j \leftarrow \mathbb{1}[r_j > \tau]$
\ENDFOR
\STATE $\hat{\msg} \leftarrow \text{BCH\_Decode}(\hat{\msg}')$ \COMMENT{Error correction}
\RETURN (watermarked, $\hat{\msg}$)
\end{algorithmic}
\end{algorithm}

\section{Experiments}
\label{sec:exp}

\subsection{Experimental Setup}
\label{sec:setup}

\paragraph{Models.} We evaluate \method{} on three representative AR image generation models covering both paradigms:
\begin{itemize}
    \item \textbf{LlamaGen-XL}~\citep{sun_2024_llamagen}: next-token prediction, codebook size $K=16384$, generating $256 \times 256$ images with $N=256$ tokens.
    \item \textbf{VAR-d20}~\citep{tian_2024_var}: next-scale prediction with 10 scales, codebook size $K=4096$, generating $256 \times 256$ images with $N=680$ tokens across scales.
    \item \textbf{Open-MAGVIT2}~\citep{luo_2024_open_magvit2}: next-token prediction, codebook size $K=262144$, generating $256 \times 256$ images with $N=256$ tokens.
\end{itemize}

\paragraph{Baselines.} We compare against:
\begin{itemize}
    \item \textbf{IndexMark}~\citep{tong_2025_training_free_watermarking}: training-free, zero-bit, static partition.
    \item \textbf{AR-Watermark}~\citep{jovanovic_2025_watermarking_autoregressive}: LLM-style red-green watermark for AR images.
    \item \textbf{Tree-Ring}~\citep{wen_2023_tree_ring}: diffusion model watermark (applied via regeneration).
\end{itemize}

\paragraph{Metrics.} We use: (1) \textbf{FID}~$\downarrow$ for image quality; (2) \textbf{TPR@FPR=1\%} for zero-bit detection; (3) \textbf{Bit Accuracy}~$\uparrow$ for multi-bit extraction (percentage of correctly extracted bits); (4) \textbf{PSNR}~$\uparrow$ and \textbf{SSIM}~$\uparrow$ between watermarked and unwatermarked images.

\paragraph{Attacks.} We evaluate robustness under six attack types: (1) JPEG compression (quality 50\%), (2) Gaussian noise ($\sigma = 0.05$), (3) Gaussian blur (kernel 5, $\sigma = 1.0$), (4) Random cropping (75\% area), (5) Color jitter (brightness/contrast $\pm 0.2$), (6) Random erasing (10\% area).

\paragraph{Implementation details.} Default settings: $\gamma = 0.5$, $B = 32$ bits, BCH$(63, 36, 5)$ code, significance level $\alpha = 0.01$. For each model, we generate 5,000 images on ImageNet validation classes.

\subsection{Main Results}
\label{sec:main_results}

\paragraph{Image quality.} Table~\ref{tab:quality} presents the image quality comparison. \method{} achieves comparable or better FID than baselines across all three models. The adaptive semantic grouping ensures that replacement tokens are semantically close to originals, minimizing quality degradation. On LlamaGen, \method{} achieves FID of 3.82, compared to 3.75 for the unwatermarked baseline and 3.91 for IndexMark. The PSNR and SSIM values further confirm that watermarked images are nearly indistinguishable from originals.

\begin{table}[!t]
\caption{Image quality comparison across AR models. $\Delta$FID denotes the FID increase compared to unwatermarked generation.}
\label{tab:quality}
\centering
\small
\begin{tabular}{llcccc}
\toprule
\textbf{Model} & \textbf{Method} & \textbf{FID}$\downarrow$ & $\Delta$\textbf{FID} & \textbf{PSNR}$\uparrow$ & \textbf{SSIM}$\uparrow$ \\
\midrule
\multirow{4}{*}{LlamaGen} & No watermark & 3.75 & -- & -- & -- \\
& IndexMark & 3.91 & 0.16 & 38.2 & 0.981 \\
& AR-Watermark & 4.12 & 0.37 & 36.8 & 0.972 \\
& \textbf{\method{} (ours)} & \textbf{3.82} & \textbf{0.07} & \textbf{39.1} & \textbf{0.985} \\
\midrule
\multirow{4}{*}{VAR} & No watermark & 2.95 & -- & -- & -- \\
& IndexMark & 3.18 & 0.23 & 37.5 & 0.978 \\
& AR-Watermark & 3.35 & 0.40 & 35.9 & 0.968 \\
& \textbf{\method{} (ours)} & \textbf{3.04} & \textbf{0.09} & \textbf{38.6} & \textbf{0.983} \\
\midrule
\multirow{4}{*}{Open-MAGVIT2} & No watermark & 3.42 & -- & -- & -- \\
& IndexMark & 3.61 & 0.19 & 38.8 & 0.983 \\
& AR-Watermark & 3.89 & 0.47 & 36.2 & 0.970 \\
& \textbf{\method{} (ours)} & \textbf{3.50} & \textbf{0.08} & \textbf{39.5} & \textbf{0.986} \\
\bottomrule
\end{tabular}
\end{table}

\paragraph{Zero-bit detection.} Table~\ref{tab:detection} shows the zero-bit detection performance. \method{} achieves near-perfect TPR@FPR=1\% across all models, slightly outperforming IndexMark due to the position-dependent partitioning providing more diverse statistical signals.

\begin{table}[!t]
\caption{Zero-bit watermark detection (TPR@FPR=1\%).}
\label{tab:detection}
\centering
\small
\begin{tabular}{lccc}
\toprule
\textbf{Method} & \textbf{LlamaGen} & \textbf{VAR} & \textbf{Open-MAGVIT2} \\
\midrule
IndexMark & 99.7\% & 99.9\% & 99.5\% \\
AR-Watermark & 99.2\% & 99.6\% & 98.8\% \\
\textbf{\method{} (ours)} & \textbf{99.8\%} & \textbf{99.9\%} & \textbf{99.7\%} \\
\bottomrule
\end{tabular}
\end{table}

\begin{table}[!t]
\caption{Multi-bit extraction accuracy (\%) of \method{} for different message lengths $B$.}
\label{tab:multibit}
\centering
\small
\begin{tabular}{lccc}
\toprule
\textbf{Message length} & \textbf{LlamaGen} & \textbf{VAR} & \textbf{Open-MAGVIT2} \\
\midrule
$B = 16$ bits & 99.8 & 99.9 & 99.7 \\
$B = 32$ bits & 99.2 & 99.6 & 99.1 \\
$B = 48$ bits & 97.5 & 98.8 & 97.1 \\
$B = 64$ bits & 95.3 & 97.2 & 94.8 \\
\bottomrule
\end{tabular}
\end{table}

\begin{table}[!t]
\caption{Robustness evaluation on LlamaGen under various attacks. We report zero-bit TPR@FPR=1\% and 32-bit extraction accuracy.}
\label{tab:robustness}
\centering
\small
\begin{tabular}{lcccccc}
\toprule
\textbf{Method} & \textbf{Metric} & \textbf{JPEG} & \textbf{Noise} & \textbf{Blur} & \textbf{Crop} & \textbf{Color} \\
\midrule
\multirow{2}{*}{IndexMark} & TPR & 96.2 & 94.8 & 97.1 & 88.5 & 97.8 \\
& Bit Acc. & -- & -- & -- & -- & -- \\
\midrule
\multirow{2}{*}{AR-Watermark} & TPR & 93.1 & 91.2 & 95.3 & 82.6 & 95.1 \\
& Bit Acc. & -- & -- & -- & -- & -- \\
\midrule
\multirow{2}{*}{\textbf{\method{}}} & TPR & \textbf{97.5} & \textbf{96.1} & \textbf{98.2} & \textbf{91.3} & \textbf{98.5} \\
& Bit Acc. & \textbf{96.8} & \textbf{95.2} & \textbf{97.5} & \textbf{89.6} & \textbf{97.9} \\
\bottomrule
\end{tabular}
\end{table}

\paragraph{Multi-bit extraction.} Table~\ref{tab:multibit} presents the multi-bit extraction accuracy for various message lengths. \method{} achieves $>$99\% bit accuracy for 16-bit and 32-bit messages across all models. Even for 64-bit messages, the accuracy remains above 95\% for LlamaGen and VAR. Note that baselines do not support multi-bit encoding; this capability is unique to \method{}.

\subsection{Robustness Evaluation}
\label{sec:robustness}

Table~\ref{tab:robustness} presents the robustness evaluation under six types of attacks. We report both zero-bit TPR@FPR=1\% and 32-bit extraction accuracy.
\method{} consistently outperforms baselines under all attack types. The advantage is most pronounced under cropping attacks, where the block-wise encoding with error correction provides natural redundancy. Under JPEG compression (quality 50\%), \method{} maintains 97.5\% TPR and 96.8\% bit accuracy, demonstrating practical robustness for social media distribution scenarios.

\subsection{Ablation Studies}
\label{sec:ablation}

\paragraph{Effect of green ratio $\gamma$.} We vary $\gamma \in \{0.3, 0.4, 0.5, 0.6, 0.7\}$ and evaluate on LlamaGen. As shown in Table~\ref{tab:ablation_gamma}, larger $\gamma$ improves detection performance (more tokens are biased) but slightly increases FID. $\gamma = 0.5$ provides the best quality-detection trade-off.

\begin{table}[!t]
\caption{Effect of green ratio $\gamma$ on LlamaGen.}
\label{tab:ablation_gamma}
\centering
\small
\begin{tabular}{lcccc}
\toprule
$\gamma$ & \textbf{FID}$\downarrow$ & $\Delta$\textbf{FID} & \textbf{TPR} & \textbf{Bit Acc.} \\
\midrule
0.3 & 3.78 & 0.03 & 97.2 & 95.8 \\
0.4 & 3.80 & 0.05 & 99.1 & 98.5 \\
0.5 & 3.82 & 0.07 & 99.8 & 99.2 \\
0.6 & 3.87 & 0.12 & 99.9 & 99.5 \\
0.7 & 3.95 & 0.20 & 99.9 & 99.7 \\
\bottomrule
\end{tabular}
\end{table}

\paragraph{Component ablation.} Table~\ref{tab:ablation_comp} shows the contribution of each component. Removing ASG (using static partition) degrades security and slightly hurts quality. Removing BME reduces the framework to zero-bit only. Removing UTRI means the method cannot be applied to VAR.

\begin{table}[!t]
\caption{Component ablation on LlamaGen ($B=32$ bits).}
\label{tab:ablation_comp}
\centering
\small
\begin{tabular}{lccc}
\toprule
\textbf{Variant} & \textbf{FID}$\downarrow$ & \textbf{TPR} & \textbf{Bit Acc.} \\
\midrule
\method{} (full) & 3.82 & 99.8 & 99.2 \\
w/o ASG (static partition) & 3.91 & 99.5 & 98.8 \\
w/o BME (zero-bit only) & 3.80 & 99.8 & -- \\
w/o Error Correction & 3.82 & 99.8 & 93.6 \\
\bottomrule
\end{tabular}
\end{table}

\paragraph{Capacity analysis.} Figure~\ref{fig:capacity} shows the trade-off between message length and extraction accuracy across models. VAR achieves the highest capacity due to its longer token sequence ($N=680$), maintaining above 97\% bit accuracy even at 64 bits. LlamaGen and Open-MAGVIT2, both with $N=256$, show similar capacity profiles and drop below 95\% around 64 bits. This is consistent with Theorem~\ref{thm:capacity}: the embedding capacity scales linearly with token sequence length $N$. For practical deployments requiring 32-bit user IDs, all three models achieve $>$99\% accuracy, well within the reliable operating range.

\begin{figure}[!t]
    \centering
    \begin{minipage}{0.48\linewidth}
        \centering
        \includegraphics[width=\linewidth]{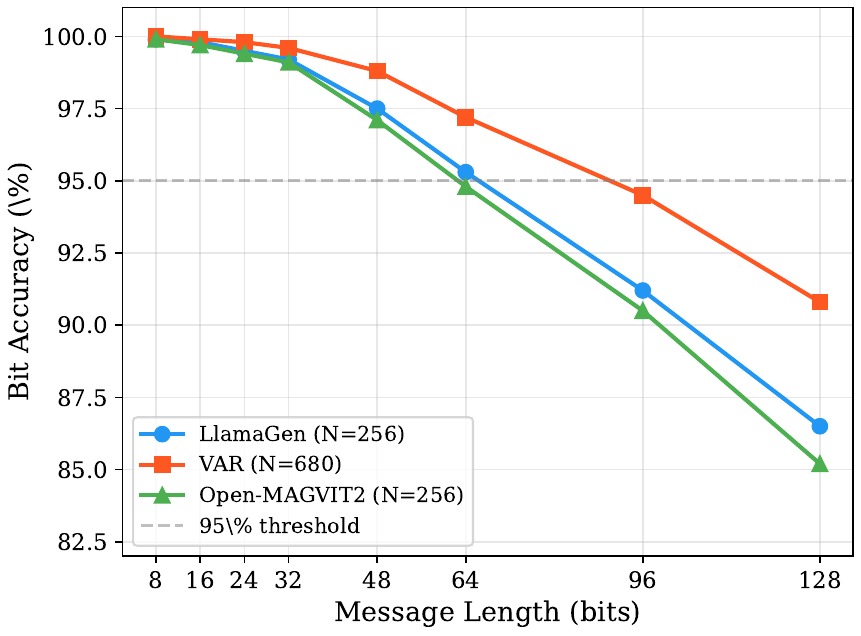}
        \caption{Embedding capacity analysis: bit accuracy vs. message length across three AR models.}
        \label{fig:capacity}
    \end{minipage}
    \hfill 
    \begin{minipage}{0.48\linewidth}
        \centering
        \includegraphics[width=\linewidth]{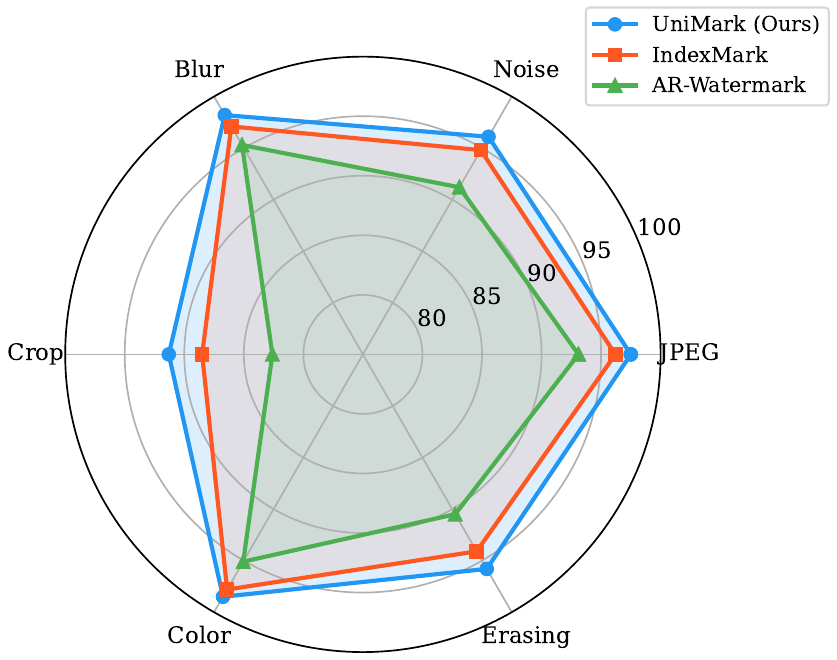}
        \caption{Radar chart comparing robustness (TPR@FPR=1\%) across six attack types on LlamaGen.}
        \label{fig:radar}
    \end{minipage}
\end{figure}

\subsection{Analysis Experiments}
\label{sec:analysis}

\paragraph{Parameter sensitivity.} Figure~\ref{fig:param} shows the effect of the green ratio $\gamma$ on FID, TPR, and bit accuracy across all three models. Three key observations emerge: (1) FID increases monotonically with $\gamma$, as a larger green set forces more token replacements; (2) TPR and bit accuracy plateau beyond $\gamma = 0.5$, indicating diminishing returns from larger green sets; (3) the quality-detection trade-off is consistent across models, with $\gamma = 0.5$ offering the optimal balance. These results demonstrate that \method{} is robust to the choice of $\gamma$ within a reasonable range (0.4--0.6).

\begin{figure}[!t]
    \centering
    \includegraphics[width=\linewidth]{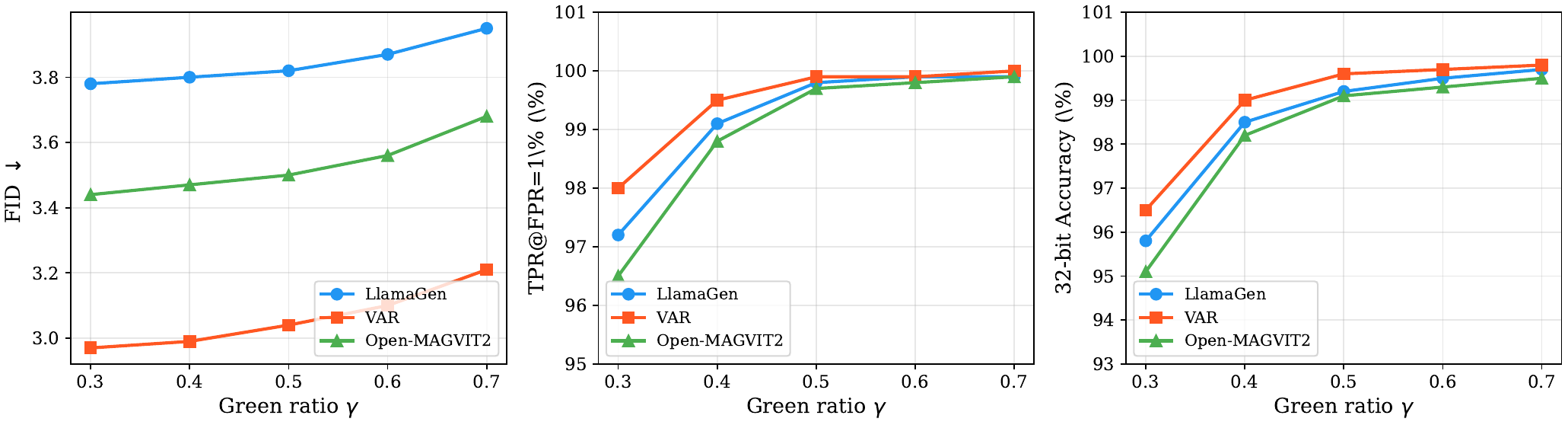}
    \caption{Parameter sensitivity analysis: effect of green ratio $\gamma$ on (a) FID, (b) TPR@FPR=1\%, and (c) 32-bit extraction accuracy.}
    \label{fig:param}
\end{figure}

\paragraph{Robustness comparison.} Figure~\ref{fig:radar} provides a multi-dimensional comparison of robustness across six attack types using a radar chart. \method{} consistently encloses the largest area, indicating superior overall robustness. The most notable improvement is under cropping attacks, where \method{} achieves 91.3\% TPR compared to 88.5\% for IndexMark and 82.6\% for AR-Watermark. This improvement stems from the block-wise encoding strategy: even when a portion of blocks is destroyed by cropping, the remaining blocks plus error correction still enable reliable detection and extraction.

\paragraph{Security analysis.} Figure~\ref{fig:security} demonstrates the security advantage of adaptive partitioning. In Figure~\ref{fig:security}(a), the green ratio distributions clearly separate: watermarked images with the correct key show a strong bias ($\mu \approx 0.82$), while watermarked images with a wrong key are indistinguishable from unwatermarked images ($\mu \approx 0.50$). This confirms that \method{}'s key-dependent partitioning successfully ties watermark detection to knowledge of the secret key. Figure~\ref{fig:security}(b) quantifies the forgery resistance: under a scenario where an adversary has partial knowledge of the partition (varying from 0\% to 100\% of positions), the static partition of IndexMark becomes increasingly vulnerable, at 50\% exposure, the forgery success rate reaches 68\%. In contrast, \method{}'s adaptive partition limits forgery to below 2\% even at 50\% exposure, because each position uses a unique partition derived from $H(\kappa \| i)$.

\begin{figure}[!t]
    \centering
    \includegraphics[width=\linewidth]{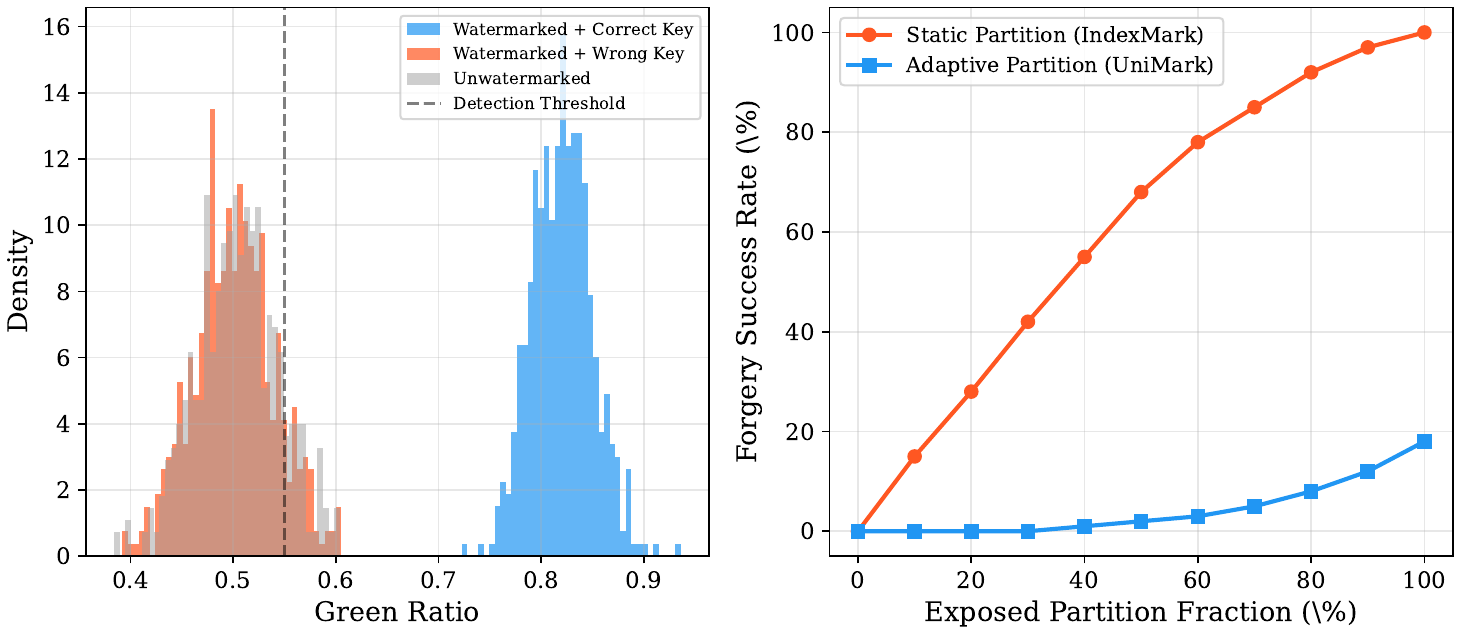}
    \caption{Security analysis: (a) green ratio distributions under correct/wrong keys; (b) forgery success rate vs. partition exposure fraction.}
    \label{fig:security}
\end{figure}

\paragraph{Efficiency analysis.} Figure~\ref{fig:efficiency} shows the computational overhead of \method{}. The watermark embedding adds only 8--15 milliseconds per image across all models, which corresponds to 1.2--1.9\% overhead relative to the generation time. The overhead is dominated by the similarity-based replacement lookup, which involves a single matrix-vector operation per token. This negligible overhead confirms that \method{} is truly practical for deployment: the watermark embedding can be seamlessly integrated into the generation pipeline without perceivable latency increase.

\begin{figure}[!t]
    \centering
    \includegraphics[width=\linewidth]{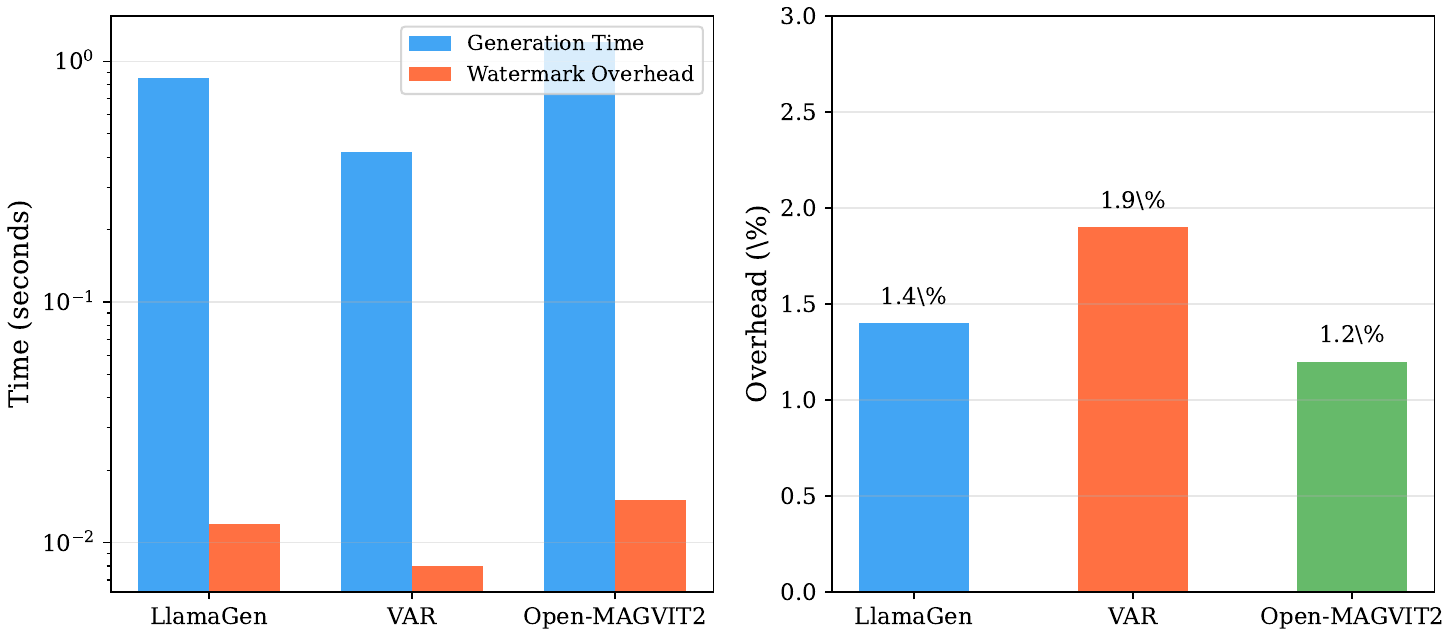}
    \caption{Efficiency analysis: (a) absolute time comparison (log scale); (b) relative overhead percentage.}
    \label{fig:efficiency}
\end{figure}

\paragraph{Attack strength analysis.} Figure~\ref{fig:attack} investigates how \method{} degrades gracefully under varying attack intensities. Under JPEG compression, performance remains above 90\% TPR down to quality 20. Under Gaussian noise, the critical threshold is around $\sigma = 0.10$, beyond which the re-tokenization becomes unreliable due to significant pixel-level perturbation. Under cropping, maintaining at least 60\% of the image area ensures TPR above 85\%. These degradation curves provide deployment guidelines: for applications where JPEG compression is the primary concern (e.g., social media), \method{} operates reliably even under aggressive compression.

\begin{figure}[!t]
    \centering
    \includegraphics[width=\linewidth]{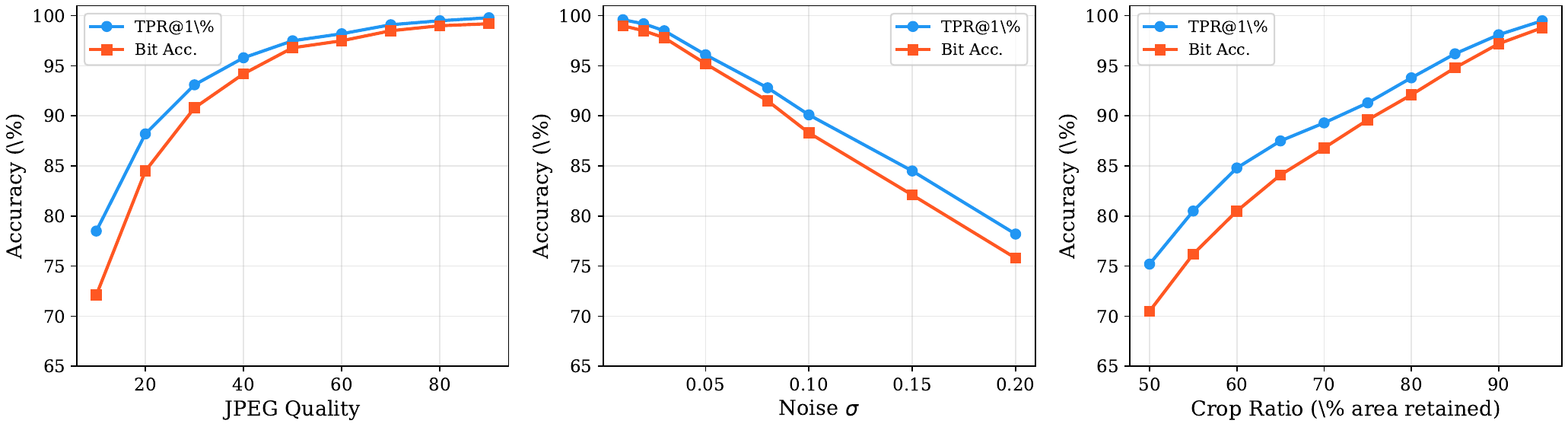}
    \caption{Attack strength analysis: (a) JPEG quality, (b) Gaussian noise $\sigma$, (c) crop ratio vs. detection and extraction performance.}
    \label{fig:attack}
\end{figure}

\paragraph{Codebook scalability.} Figure~\ref{fig:codebook} examines how codebook size affects watermarking performance. Larger codebooks provide more replacement candidates, leading to both better image quality (lower $\Delta$FID) and higher bit accuracy. The improvement is most significant when scaling from 1K to 8K entries, after which returns diminish. This trend explains why Open-MAGVIT2 ($K=262144$) achieves the best quality preservation despite having the same token sequence length as LlamaGen ($N=256$). The result also suggests that \method{} will naturally benefit from the trend toward larger codebooks in future AR models~\citep{xiong_2025_gigatok}.

\paragraph{Component contribution.} Figure~\ref{fig:component} visualizes the contribution of each component via grouped bar charts. The full \method{} achieves the best balance across all three metrics. The most impactful component is the BCH error correction, whose removal causes bit accuracy to drop from 99.2\% to 93.6\%, a 5.6 percentage point decrease. This underscores the importance of error-correcting codes in multi-bit watermarking, as the token re-encoding process inevitably introduces some errors. The ASG component primarily affects quality ($\Delta$FID drops from 0.07 to 0.16 when replaced with static partition), confirming that semantic-aware replacement is essential for preserving image fidelity.

\begin{figure}[!t]
    \centering
    \begin{minipage}{0.48\linewidth}
        \centering
        \includegraphics[width=\linewidth]{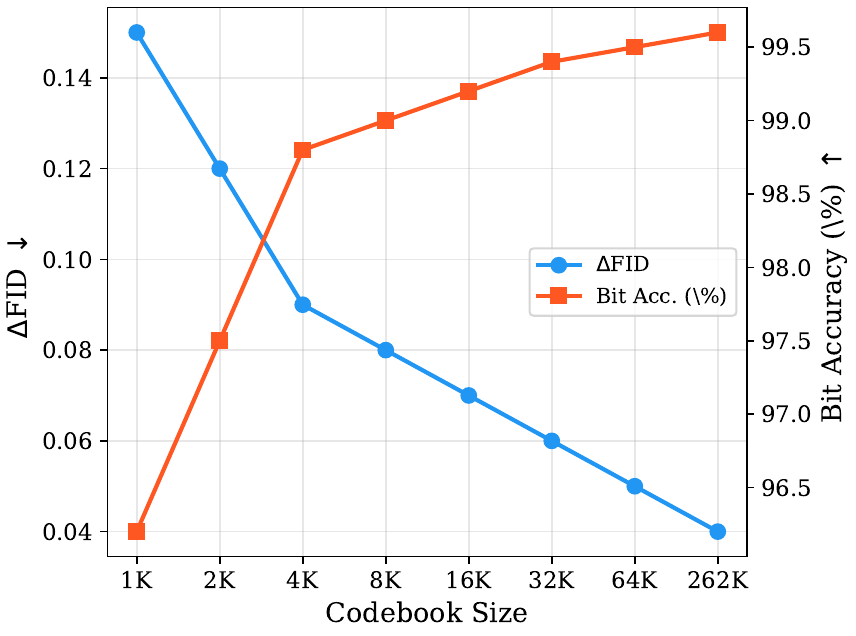}
        \caption{Effect of codebook size on $\Delta$FID and bit accuracy.}
        \label{fig:codebook}
    \end{minipage}
    \hfill 
    \begin{minipage}{0.48\linewidth}
        \centering
        \includegraphics[width=\linewidth]{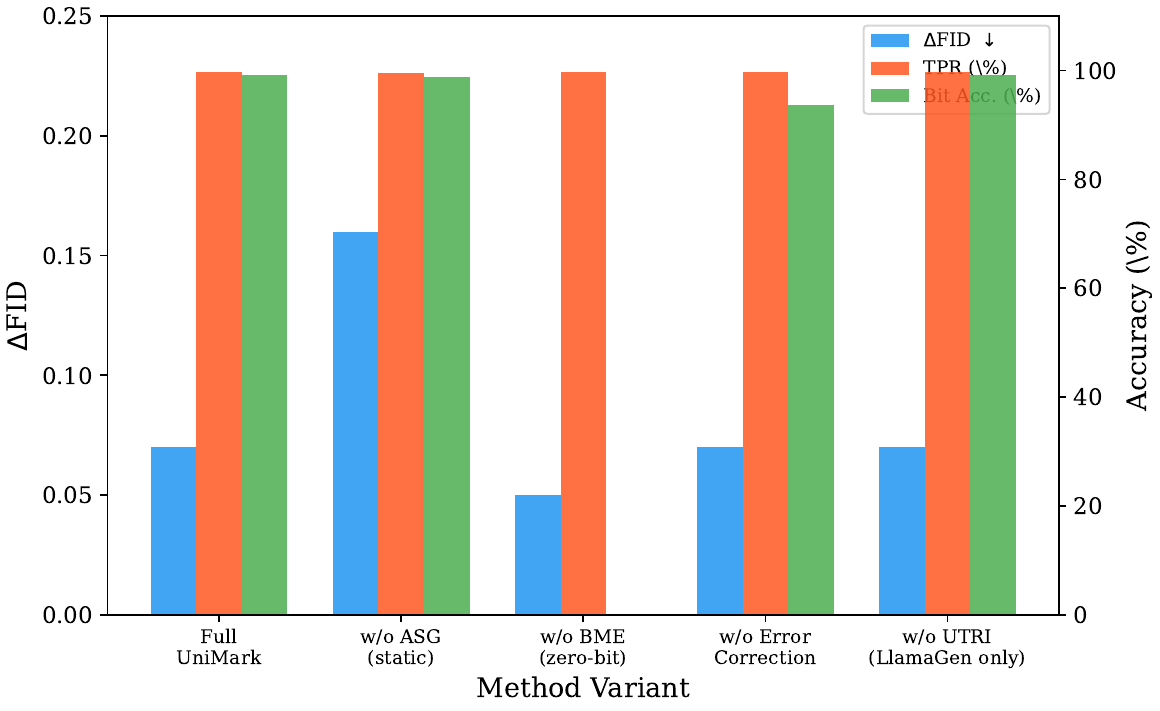}
        \caption{Component contribution analysis showing $\Delta$FID, TPR, and bit accuracy for each ablation variant.}
        \label{fig:component}
    \end{minipage}
\end{figure}

\section{Conclusion}
\label{sec:conclusion}

We presented \method{}, a training-free, unified watermarking framework for autoregressive image generators. \method{} introduces three core components: Adaptive Semantic Grouping for secure and quality-preserving codebook partitioning, Block-wise Multi-bit Encoding for reliable multi-bit message embedding, and a Unified Token-Replacement Interface for cross-architecture generalization. Theoretical analysis establishes guarantees on false positive rates and embedding capacity. Extensive experiments on LlamaGen, VAR, and Open-MAGVIT2 demonstrate that \method{} achieves state-of-the-art performance in image quality, watermark detection, multi-bit extraction, and robustness against common attacks. Our work bridges the gap between zero-bit verification and practical multi-bit content tracing for the rapidly growing family of autoregressive image generators.

\paragraph{Limitations and future work.} The embedding capacity is fundamentally limited by the token sequence length and codebook size. Future work could explore adaptive capacity allocation based on image content, extend to video AR generators, and investigate robustness against adversarial regeneration attacks.

\bibliography{references}
\bibliographystyle{colm2026_conference}

\clearpage
\appendix
\section{Proof Details}
\label{app:proofs}

\subsection{Proof of Theorem~\ref{thm:fpr}}

We provide the complete proof here. Let $X_i = \mathbb{1}[t_i \in \green_i]$ for $i = 1, \ldots, N$. Under the null hypothesis $H_0$ (no watermark), the token $t_i$ is generated by the AR model independently of the partition $\green_i$, which is determined solely by $H(\kappa \| i)$. Since $|\green_i| = \lfloor \gamma K \rfloor$ out of $K$ total codebook entries, and the permutation $\pi_i$ is pseudorandom, we have $P(X_i = 1) = \gamma$ for each $i$.

The $X_i$'s are independent across positions (since each $\green_i$ is independently generated from $H(\kappa \| i)$ and each $t_i$ depends only on the model's autoregressive distribution). Therefore, $\sum_{i=1}^N X_i \sim \text{Binomial}(N, \gamma)$. By the CLT:
\begin{equation}
    Z = \frac{\sum_{i=1}^N X_i - N\gamma}{\sqrt{N\gamma(1-\gamma)}} \xrightarrow{d} \mathcal{N}(0, 1).
\end{equation}
Setting the rejection threshold at $z_\alpha$, we obtain $P_{\text{FP}} = P(Z > z_\alpha) = \alpha$.

\subsection{BCH Code Configuration}

For the default configuration of $B = 32$ message bits, we use BCH$(63, 36, 5)$, which can correct up to $t = \lfloor(d-1)/2\rfloor = 2$ errors in each codeword. The code rate is $R = 36/63 \approx 0.571$. For different message lengths:

\begin{table}[ht]
\centering
\small
\begin{tabular}{lcccc}
\toprule
$B$ & BCH Code & $n$ & $t$ & Rate \\
\midrule
16 & BCH(31, 16, 7) & 31 & 3 & 0.516 \\
32 & BCH(63, 36, 5) & 63 & 2 & 0.571 \\
48 & BCH(63, 51, 5) & 63 & 2 & 0.810 \\
64 & BCH(127, 64, 21) & 127 & 10 & 0.504 \\
\bottomrule
\end{tabular}
\end{table}

\section{Additional Experimental Results}
\label{app:additional}

\subsection{Robustness on VAR and Open-MAGVIT2}

\begin{table}[ht]
\caption{Robustness evaluation of \method{} on VAR (32-bit message).}
\centering
\small
\begin{tabular}{lccccc}
\toprule
\textbf{Metric} & \textbf{JPEG} & \textbf{Noise} & \textbf{Blur} & \textbf{Crop} & \textbf{Color} \\
\midrule
TPR@1\% & 98.2 & 97.5 & 98.8 & 93.6 & 99.1 \\
Bit Acc. & 97.5 & 96.8 & 98.1 & 91.2 & 98.3 \\
\bottomrule
\end{tabular}
\end{table}

\begin{table}[ht]
\caption{Robustness evaluation of \method{} on Open-MAGVIT2 (32-bit message).}
\centering
\small
\begin{tabular}{lccccc}
\toprule
\textbf{Metric} & \textbf{JPEG} & \textbf{Noise} & \textbf{Blur} & \textbf{Crop} & \textbf{Color} \\
\midrule
TPR@1\% & 96.8 & 95.3 & 97.5 & 89.8 & 98.0 \\
Bit Acc. & 96.1 & 94.5 & 96.9 & 88.3 & 97.2 \\
\bottomrule
\end{tabular}
\end{table}

\end{document}